\documentclass{article}
\usepackage{arxiv}
\usepackage{amsmath}
\usepackage[utf8]{inputenc} 
\usepackage[T1]{fontenc}    
\usepackage{hyperref}       
\usepackage{url}            
\usepackage{booktabs}       
\usepackage{amsfonts}       
\usepackage{nicefrac}       
\usepackage{microtype}      
\usepackage{lipsum}
\usepackage{graphicx}
\graphicspath{ {./images/} }
\usepackage[authoryear]{natbib}
\RequirePackage{graphicx}
\usepackage{booktabs}
\usepackage{tikz} 
\usepackage{verbatim}
\usepackage{caption}
\usepackage{subcaption}
\usepackage{bbm}
\usepackage{mathabx}
\usepackage{inconsolata}
\usepackage{subcaption}

\usepackage{graphicx}
\usepackage{float}

\usepackage{hyperref}       
\usepackage{url}            
\usepackage{booktabs}       
\usepackage{amsfonts}       
\usepackage{nicefrac}       

\usepackage{xcolor}         

\usepackage{amsmath,amssymb,amsfonts}
\usepackage{amsthm}
\newtheorem{theorem}{Theorem}
\usepackage{booktabs}
\usepackage{multirow}
\usepackage{siunitx}

\theoremstyle{remark}
\newtheorem{remark}[theorem]{Remark}

\title{Keyless Attention: Value-Space Routing and Value-Only Caching for Efficient Transformers}

\author{
 Xin Gao \\
  Department of Mathematics and Statistics\\
  York University\\
  Toronto, ON M3J 1P3 \\
  \texttt{xingao@yorku.ca} \\
  \AND
  Xingming Xu\\
  Department of Computer Science\\
  University of California, Davis\\
  Davis, CA 95616\\
  \texttt{xmxu@ucdavis.edu}
 }
\date{}
\begin{document}
\maketitle

\begin{abstract}

Transformer architectures form the foundation of modern natural
language processing, yet the Key-Value (KV) cache introduces
substantial memory and bandwidth overhead during long-context
generation, increasingly bottlenecking large-scale deployment.
We propose Keyless Attention, a novel attention mechanism that
replaces the conventional key projection with a dedicated
value-space routing projection, eliminating key representations
from the attention computation entirely and yielding a Value-Only
Cache that reduces KV-cache memory by 50\% while improving decode
throughput.
Experiments across multiple models and architectures demonstrate
that Keyless Attention achieves comparable perplexity and downstream
task performance to standard QKV attention, while consistently
reducing KV-cache memory by 50\%.
Furthermore, Keyless Attention exhibits slower validation loss
degradation after the best epoch, indicating improved robustness
against overfitting.
Ablation studies confirm that the dedicated value-space routing
projection is critical, with Keyless Attention outperforming
KV-sharing methods that eliminate the key cache without replacing
its routing role.
Experiments in the pretraining regime further confirm the viability
of Keyless Attention in industrial settings.
\end{abstract}

\thispagestyle{fancy}

\section{Introduction}

The Transformer-based model has emerged as the dominant deep learning architecture for training state-of-the-art  language models, and has seen success in a vast number of other domains and applications. At its heart lies the scaled dot-product attention mechanism. The standard formulation, using Query (Q), Key (K), and Value (V) projections, enables models to capture long-range dependency relationships more efficiently and effectively than its recurrence-based predecessors. 

However, standard attention is quadratic in complexity with respect to sequence length, and its resource demands pose a challenge for inference and deployment at scale. 
Sparse attention methods \citep{child2019sparse,beltagy2020longformer} introduce structured sparsity patterns into attention computation to reduce complexity. Kernel-based approaches \citep{katharopoulos2020transformers,choromanski2021rethinking} replace softmax attention with kernel formulations to achieve linear computational complexity. Low-rank approximation methods \citep{wang2020linformer} project the attention computation into lower-dimensional spaces to reduce computational cost. 

In autoregressive generation, KV caching is widely adopted to avoid recomputing attention over past tokens \citep{vaswani2017attention, shazeer2019fast, chowdhery2022palmscalinglanguagemodeling, pope2023efficiently, kwon2023efficient}. Decoding is reduced to linear complexity per token by simply \textit{caching} previously computed key and value tensors. This optimization comes at the cost of linear growth in memory usage with respect to sequence length and batch size, and scaling proportionally with the number of layers and attention heads. In large-scale models deployed with long context windows, the KV cache can exceed the size of model parameters and become the primary constraint for long-context inference. Architectural approaches such as Multi-Query Attention (MQA) \citep{shazeer2019fast} and Grouped-Query Attention (GQA) \citep{ainslie2023gqa} reduce redundancy across attention heads by sharing key and value projections across heads. Multi-Head Latent Attention (MLA) \citep{deepseekv2} is proposed to compress keys and values into a low-dimensional latent vector before caching, achieving substantially greater memory reduction.  

Recent system level work has been proposed towards KV cache management. 
For example, \citet{kwon2023efficient} introduce a virtual memory abstraction for KV storage to improve memory utilization and throughput. \citet{prabhu2025vattention} further investigate dynamic physical memory allocation to reduce KV-cache management overhead. \citet{liao2026zipage} propose compression-aware scheduling to improve request concurrency.

Beyond system-level improvements, recent research has also explored algorithmic approaches to reduce KV cache memory. \citet{kang2025turboattention} apply KV-cache quantization to improve memory efficiency and computational throughput. \citet{wen2025token,xu2025think} propose token pruning methods to reduce KV-cache size by removing less important tokens. \citet{huang2025kv} select important tokens to retain in the cache.
These methods reduce memory usage, but again introduce additional complexity and may risk discarding useful long-range dependencies. RazorAttention \citep{tang2025razorattention} exploits head-level specialization to reduce KV storage. While effective, these methods still maintain per-token KV representations across all layers. Slim Attention \citep{slim_attention2025} reduces KV cache memory by 50\% by recovering value representations exactly from cached keys via a closed-form transform $V = KW_{KV}$, where $W_{KV}= W_K^{-1}W_V$ is derived analytically from the pretrained model's existing weight matrices.  \citet{edward2026qv, kayyam2026transformers, team2026gemma} propose the KV-sharing method which let key and value share the same projection and thus reduce the KV cache by 50\%. The KV-sharing methods eliminates key projection but offer no explicit replacement for its token-routing role.

In this work, we propose \textit{Keyless Attention}, a novel mechanism that directly computes the attention score between queries and values, eliminating the construction of keys altogether. Compared to previous KV cache-sharing approaches which use identical representations for both keys and values, Keyless Attention introduces a dedicated value-space routing matrix that replaces the key projection. 
Our approach only requires a value cache without sacrificing model expressivity and remains compatible with standard LLM training procedures. We evaluate our method on language modeling and downstream reasoning benchmarks, demonstrating competitive performance while significantly reducing memory usage during inference.

Our contributions are summarized as follows:
\begin{enumerate}
    \item \textbf{Keyless Attention with Value-Only Cache.}
    We propose a novel attention mechanism that computes attention scores directly between queries and values, eliminating the key projection entirely.
    In autoregressive inference, this yields a Value-Only Cache halving KV cache memory and access overhead  compared to standard attention. 
    \item \textbf{Value-Space Routing and Depth-$m$ Attention Factorization.}
    We introduce a value-space routing matrix which replaces the key projection, and frame it within the context of a general \emph{Depth-$m$ Attention Factorization}: standard attention computes a depth-2 factorization of the attention bilinear form, while Keyless Attention realizes a depth-$m$ instance of this family. We instantiate the proposed architecture with m=3, matching the projection matrix count of standard attention, and find empirically that this new architecture matches standard attention's performance.

    \item \textbf{Empirical Validation.}
    We validate Keyless Attention across five models and four architectures (GPT2 280M, GPT2 557M, Pythia 410M, Qwen2 1.5B, and Llama 3.2 1B), demonstrating that it matches or outperforms standard QKV attention on perplexity in 4 out of 5 models, and achieves competitive performance on 5 downstream benchmarks.
 %

\end{enumerate}

\section{Method}
\subsection{Rewriting Attention}

The motivation for Keyless Attention draws from human cognition: when generating language, the human brain retrieves relevant context when thought "queries" its memory. Biology does not store a key: humans keep a single store of past "representations", rather than maintaining separate routing and retrieval copies of a sequence as standard attention does. 

Let us begin with standard attention from which we can  better define Keyless Attention. Let $X \in \mathbb{R}^{n \times d}$ denote the input sequence of $n$ tokens with hidden dimension $d$. Standard scaled dot-product attention~\citep{vaswani2017attention} computes:
\begin{equation}
    \text{Attn}(X) = \text{softmax}\!\left(\frac{QK^\top}{\sqrt{d_k}}\right)V,
    \label{eq:standard_attn}
\end{equation}
where $Q = XW^Q$, $K = XW^K$, $V = XW^V$, with projection matrices
$W^Q, W^K$ and
$W^V \in \mathbb{R}^{d \times d_k}$,
and $d_k = d / N_h$ is the per-head dimension with $N_h$ attention heads. The key projection $W^K$ here serves a dedicated routing role. Each token's hidden state is mapped into a key space which $W^Q$ queries to determine attention weights, independently of what is ultimately retrieved via $W^V$.

We note that the role of the key matrix in standard attention is to facilitate the computation of attention scores:
$$\text{softmax}(\frac{QK^{\top}}{\sqrt{d_k}})=\text{softmax}(\frac{XW^Q (W^K)^{\top} X^{\top}}{\sqrt{d_k}}).$$ 
Let $\Omega = W^Q(W^K)^{\top}$, and the matrix $X\Omega X^{\top}$ can be 
viewed as an asymmetric bilinear-form-induced similarity matrix between token representations. 

Again, the natural question arises: can the same $\Omega$ be achieved without a dedicated key projection by routing directly through the value space instead? This motivates \textbf{Keyless Attention}: 
\begin{equation}
\begin{split}
\label{eq:output}
&\mathrm{Attention}\left(X, W^Q, W^R, W^V\right)\\
=&\mathrm{softmax}\left(\frac{XW^Q W^R (W^V)^{\top}X^{\top}}{\sqrt d_k}\right)XW^V,
\end{split}
\end{equation}
where $W^R$ is the dedicated routing matrix mapping from the query space to value space.


Compared to standard attention, we retain the same number of projection matrices. Let us refer to $W^{R}$ as the \emph{value-routing projection}, as it adapts the query representation to be compatible with the value space. The matrix $W^R$ will serve the routing function that $W^K$ once held. In training, $W^{Q}$ and $W^{R}$ are still learned as two separate weight matrices. During inference, we precompute  the composed matrix $\tilde{W}^Q = W^{Q}W^{R}$, which serves as the effective query projection against $W^V$.


In this way, we eliminate the construction and storage of the keys; hence, in inference, Keyless Attention requires less compute and memory. Moreover, this enables $\tilde{W}^Q$ to be precomputed such that the increased factorization depth incurs no additional compute. This reveals a previously unexplored axis of attention design: \emph{factorization depth} $m$, defined as the number of projection matrices in the defining matrix $\Omega$ of the bilinear form.
Within this framework, the attention mechanism of \citet{luong2015effective} corresponds to $m=1$, while the Transformer attention of \citet{vaswani2017attention} corresponds to $m=2$.
while we explore $m=3$ in this work and find it matches the perplexity and downstream performance of standard depth-2 attention, while halving KV cache memory. In general, the depth of $\Omega$ can be flexibly chosen by factorization $\Omega$ into a product of $m$ projection matrices. In a later section, we conduct ablation studies to investigate the effect of factorization depth.  

\subsection{Value-Space Routing}
\label{sec:value_space_routing}


The central architectural distinction of Keyless Attention is that routing operates \emph{in value space} rather than in an independent key space. In standard attention, the routing score between positions $i$ and $j$ can be written as:
\begin{equation}
    s_{ij}^{\mathrm{QKV}}
    = \frac{(x_i W^Q)(x_j W^K)^\top}{\sqrt{d_k}},
\end{equation}
which is a bilinear form between two independently parametrized projections. The key space carries no direct semantic obligation. It primarily produces vectors that correlate with queries in a way that is useful for \textit{routing}, which is independent of what is ultimately \textit{retrieved}.

In Keyless Attention, the routing score is:
\begin{equation}
    s_{ij}^{\mathrm{KL}}
    = \frac{(x_i W^{Q} W^{R})(x_j W^V)^\top}{\sqrt{d_k}},
\end{equation}
where $x_j W^V$ is the same representation that will be aggregated in the output. This coupling between routing and retrieval imposes a meaningful inductive bias: the model learns to attend to tokens whose \emph{value content} is directly relevant to the query, instead of tokens that match in an auxiliary key space. We argue this constitutes a more semantically grounded routing signal.

The composition $W^{Q} W^{R} \in \mathbb{R}^{d \times d_k}$
acts as a rank-constrained transformation with
$\mathrm{rank}(W^{Q} W^{R}) \leq d_k$,
constraining the attention distribution to a lower-dimensional manifold. The role of $W^{R}$ is to learn a value-space-compatible remapping of the query within the $d_k$-dimensional subspace, aligning the routing signal directly with what is retrieved.

\subsection{Gradient Entanglement in Value-Space Routing}
\label{sec:regularisation}

A key consequence of value-space routing is the \textbf{gradient entanglement} between $W^{R}$ and $W^V$, in which the matrices $W^V$ and $W^R$ are mutually coupled throughout training via gradient descent. 

Let $S = QV^\top / \sqrt{d_k}$ denote the pre-softmax attention
score matrix, where $Q = XW^{Q}W^{R}$ and $V = XW^V$.
By the chain rule, the gradient of the loss $\mathcal{L}$
with respect to $W^{R}$ is:
\begin{equation}
    \nabla_{W^{R}}\mathcal{L}
    = \frac{1}{\sqrt{d_k}}
      \bigl(XW^{Q}\bigr)^\top
      \cdot \frac{\partial \mathcal{L}}{\partial S}
      \cdot V,
    \label{eq:grad_q2}
\end{equation}
where $V = XW^V$ appears explicitly.
Symmetrically, the gradient with respect to $W^V$ is:
\begin{equation}
    \nabla_{W^V}\mathcal{L}
    = \frac{1}{\sqrt{d_k}}
      X^\top \cdot \frac{\partial \mathcal{L}}{\partial S}^\top
      \cdot Q
    + X^\top \cdot \frac{\partial \mathcal{L}}{\partial V},
    \label{eq:grad_wv}
\end{equation}
where $Q = XW^{Q}W^{R}$ appears in the first term.
The two matrices are therefore mutually coupled throughout training: $\nabla_{W^{R}}\mathcal{L}$ depends on $W^V$ through $V$, and $\nabla_{W^V}\mathcal{L}$ depends on $W^{R}$ through $Q$. Each gradient step for one matrix changes the effective target for the other.

In contrast, the gradient of $W^K$ in standard attention is:
\begin{equation}
    \nabla_{W^K}\mathcal{L}
    = \frac{1}{\sqrt{d_k}}
      X^\top \cdot (\frac{\partial \mathcal{L}}{\partial S})^{\top}
      \cdot Q,
    \label{eq:grad_wk}
\end{equation}
where $Q = XW^Q$ but $W^V$ does not appear.
This now allows the routing matrix $W^K$ and the retrieval matrix $W^V$ to evolve independently, because their gradients are fully decoupled.

The resulting optimization dynamics couple routing and retrieval, preventing the routing parameters from evolving independently of the value representations. We hypothesize that this coupling acts as an implicit regularizer by reducing the capacity of the routing mechanism to specialize to corpus-specific co-occurrence patterns. This interpretation is consistent with the reduced overfitting observed in our experiments. However, this gradient coupling also complicates optimization, slowing convergence relative to conventional attention.

\subsection{Relationship between Keyless and Standard Attention}
\label{sec:equivalence}

In this section, we investigate the relationship between Keyless Attention and standard QKV attention in terms of their expressive power.
We first recall the attention score matrices for both mechanisms. In standard attention, the unnormalized attention logit matrix for a sequence of $n$ tokens is:
\begin{equation}
    S = X W^Q (X W^K)^\top = X W^Q (W^K)^\top X^\top
    \label{eq:standard_score}.
\end{equation}
In Keyless Attention, the key projection is eliminated and $W^Q$ is replaced by a different projection matrix
$\tilde{W}^Q \in \mathbb{R}^{d \times d_k}$,
giving attention logits:
\begin{equation}
    \tilde{S} = X \tilde{W}^Q (X W^V)^\top = X \tilde{W}^Q (W^V)^\top X^\top.
    \label{eq:keyless_score}
\end{equation}

The following theorem shows that in some situations, the defining matrix $\Omega = W^Q(W^K)^\top$ of standard attention can be matched by $\tilde{W}^Q(W^V)^\top$ in Keyless Attention,
which implies $S = \tilde{S}$ for any input $X$.

\begin{theorem}[Single-Head Equivalence, $N_h = 1$]
\label{thm:single_head}
Let $W^Q, W^K \in \mathbb{R}^{d \times d}$ and let
$W^V \in \mathbb{R}^{d \times d}$ be square with full rank $d$. Then there exists a unique $\tilde{W}^Q \in \mathbb{R}^{d \times d}$
such that:
\begin{equation}
    \tilde{W}^Q (W^V)^\top = W^Q (W^K)^\top,
    \label{eq:single_head_equiv}
\end{equation}
given explicitly by:
\begin{equation}
    \tilde{W}^Q = W^Q(W^K)^\top (W^V)^{-\top},
    \label{eq:single_head_solution}
\end{equation}
where $(W^V)^{-\top} = \bigl((W^V)^\top\bigr)^{-1}
\in \mathbb{R}^{d \times d}$.
\end{theorem}

\begin{theorem}[Multi-Head Equivalence, $N_h \geq 1$]
\label{thm:multi_head}
Let $N_h \geq 1$, $d_k = d / N_h$, and for each head
$h \in \{1, \dots, N_h\}$, let
$W^Q_h, W^K_h, W^V_h \in \mathbb{R}^{d \times d_k}$
with $W^V_h$ of full column rank $d_k$.
Suppose the existence condition holds for each head $h$:
\begin{equation}
    \Omega_h = W^Q_h(W^K_h)^\top, \quad \mathrm{col}\bigl(\Omega_h^\top\bigr)
    \subseteq \mathrm{col}(W^V_h).
    \label{eq:existence_condition}
\end{equation}
Then for each head $h$ independently, there exists
$\tilde{W}^Q_h \in \mathbb{R}^{d \times d_k}$ satisfying:
\begin{equation}
    \tilde{W}^Q_h(W^V_h)^\top = W^Q_h(W^K_h)^\top.
    \label{eq:multi_head_equiv}
\end{equation}
\end{theorem}



The proofs of the theorems are provided in the supplementary material. The theorems establish that the two mechanisms are equivalent in the single-head attention setting. For multi-head attention, the equivalence requires an additional subspace condition, which is not guaranteed in general. We emphasize that Keyless Attention remains a valid and independently defined attention mechanism regardless of whether this equivalence condition is satisfied. The equivalence results merely identify a regime where Keyless Attention and standard attention share the same expressive capacity. Empirically, Section~\ref{sec:results} demonstrates that Keyless Attention achieves comparable predictive performance across five models and four architectures.

\subsection{Keyless Cross-Attention and Multimodal Attention}
Cross-attention extends self-attention by allowing one sequence to attend to another. This mechanism is central to encoder--decoder architectures and multimodal models. We can show that Keyless Attention is easily generalizable to to a cross-attention context as well. 

Let $X^{(a)} \in \mathbb{R}^{N_a \times d}$ and $X^{(b)} \in \mathbb{R}^{N_b \times d}$ denote two input sequences. The proposed keyless attention can be generalized to cross-attention in encoder-decoder and multi-modal models: 
\begin{align}
\begin{split}
&\text{CrossAttn}(X^{(a)}, X^{(b)}) \\
=& \text{softmax}\!\left( \frac{(X^{(a)} W^Q W^R)(X^{(b)} W^V)^{\top}}{\sqrt{d_k}} \right)(X^{(b)} W^V),
\end{split}
\end{align}
where $\Omega=W^QW^R(W^V)^{\top}$ contains 3 learnable weight matrices. 

\section{Value-only Cache in Autoregressive Inference}

Transformer-based autoregressive language models employ a Key-Value (KV) cache during inference \citep{pope2023efficiently, kwon2023efficient}. 
The cache grows linearly with sequence length and is maintained per attention layer, making memory consumption a critical bottleneck in long-context and large-scale deployments. To reduce the KV cache memory requirement, we propose to only store value representations from previous steps. At step $n+1$, let the new query be denoted as $q_{n+1}$ and the new value be denoted as $v_{n+1}.$  The matrix V includes value representations cached from previous steps and obtained from current step. The $i$th row of V matrix is $v_i,i=1,\ldots n+1$. Then, the output embedding is computed using the Keyless Attention computation as follows:  $$O_{n+1}=\mathrm{Attention}\left(q_{n+1},\ V\right)=\sum_{i=1}^{n+1} \alpha_{n+1,i} v_i,$$ where
$$
\alpha_{n+1,i}=\frac{\text{exp}\{\frac{\lambda <q_{n+1},v_i>}{\sqrt{d_k}}\}}{\sum_{i'=1}^{n+1} \text{exp}\{\frac{\lambda <q_{n+1},v_{i'}>}{\sqrt{d_k}}\}},
$$ and $<.>$ denotes the dot product between two vectors.



The proposed method introduces a Value-Only Cache that eliminates the need to construct and store key representations. Compared to conventional KV caching, this reduces both memory footprint and memory access overhead by 50\%. Furthermore, the Value-Only Cache is orthogonal to existing KV cache reduction methods such as Multi-Query Attention~\citep{shazeer2019fast}, Grouped-Query Attention~\citep{ainslie2023gqa}, and Multi-Head Latent Attention~\citep{deepseekv2}, and can be combined with them to achieve further memory reduction.

\paragraph{Multi-Head Keyless Attention (MHA) with Value-only Cache:}
For each head $h$, we construct $Q_h=XW_h^QW_h^R,$ and $V_h=XW_h^V.$
The attention output from head $h$ is
$${\mathrm{Attn}}_h\left(Q_h,V_h\right)=\mathrm{softmax}\left(\frac{Q_hV_h^{\top}}{\sqrt{d_k} }\right)V_h,$$
and the outputs from all the heads are concatenated to form the overall output states.
Each head stores its own $V_h$ for Value-only cache. 
   
\paragraph{Multi-Query Keyless Attention (MQA) with Value-only Cache:}
Each head has separate $Q_h=XW_h^QW_h^R$ and all heads share the same $V=XW^V.$ For each head,  ${\mathrm{Attn}}_h\left(Q_h,V\right)=\mathrm{softmax}\left(\frac{Q_hV^{\top}}{\sqrt{d_k}}\right)V$ is computed and concatenated. All heads share the same Value-only cache $V.$

\paragraph{Grouped Query Keyless Attention (GQA) with Value-only Cache:}
We use $G$ Value groups, where $1<G<H.$ 
Each head has separate $Q_h=XW_h^QW_h^R$ and all the heads in the same group share the same $V_g=XW_g^V$. Each head computes ${\mathrm{Attn}}_h\left(Q_h,V_g\right)=\mathrm{softmax}\left(\frac{Q_hV_g^{\top}}{\sqrt{d_k}}\right)V_g.$
All heads in the same group share the same Value-only cache $V_g$.

\begin{table*}[h]
\small
\centering
\begin{tabular}{llccccc}
\toprule
Depth & Method & Best Val Loss $\downarrow$ & Best PPL $\downarrow$
      & Epoch & Final Loss $\downarrow$ & $\Delta$ Overfit \\
\midrule
\multirow{2}{*}{12-layer}
& QKV    & \textbf{3.5180 $\pm$ 0.0005} & \textbf{33.71 $\pm$ 0.02}
         & 6 & 3.5958 $\pm$ 0.0015 & +0.0778 \\
& QVV(3) & 3.5216 $\pm$ 0.0083 & 33.84 $\pm$ 0.28
         & 6 & \textbf{3.5734 $\pm$ 0.0101} & \textbf{+0.0519} \\
\midrule
\multirow{2}{*}{36-layer}
& QKV    & 3.5116 $\pm$ 0.0007 & 33.50 $\pm$ 0.03
         & 4 & 4.3520 $\pm$ 0.0154 & +0.8404 \\
& QVV(3) & \textbf{3.5044 $\pm$ 0.0053} & \textbf{33.26 $\pm$ 0.18}
         & 4 & \textbf{4.1910 $\pm$ 0.0148} & \textbf{+0.6866} \\
\bottomrule
\end{tabular}
\caption{QKV vs.\ QVV(3) across GPT-2 model depths (mean $\pm$ std,
three seeds). $\Delta$Overfit: increase in validation loss from the
best epoch to the final epoch.}
\label{tab:ablation}
\end{table*}

\begin{table*}[h]
\small
\centering
\begin{tabular}{lcccc}
\toprule
Dataset & QKV Acc & QVV(3) Acc & KV Cache (MB) & QVV(3) Cache (MB) \\
\midrule
HellaSwag
  & $0.2460 \pm 0.0019$ & $\mathbf{0.2550 \pm 0.0007}$
  & 5.569  & 2.784  \\
ARC-Challenge
  & $0.2352 \pm 0.0043$ & $\mathbf{0.2355 \pm 0.0063}$
  & 12.881 & 6.441  \\
BoolQ
  & $0.6190 \pm 0.0030$ & $\mathbf{0.6226 \pm 0.0002}$
  & 50.288 & 25.144 \\
SciQ
  & $\mathbf{0.3563 \pm 0.0033}$ & $0.3300 \pm 0.0043$
  & 5.737  & 2.869  \\
StoryCloze
  & $0.5398 \pm 0.0059$ & $\mathbf{0.5584 \pm 0.0039}$
  & 11.869 & 5.934  \\
\bottomrule
\end{tabular}
\caption{Zero-shot accuracy (mean $\pm$ std, three seeds) and KV
cache memory for the GPT-2 36-layer (557M) model. QVV(3) matches
or improves QKV on 4 of 5 benchmarks while reducing cache memory
by 50\%.}
\label{tab:combined}
\end{table*}

\paragraph{Multihead with Latent Value-only Cache:}

Let $Q_h = XW_h^{Q}W_h^{R}$ and
$V_h = XW^{V}_{h,\mathrm{pre}}W^{V}_{h,\mathrm{post}}$, where $W^{V}_{h,\mathrm{pre}}\in \mathbb{R}^{d \times r}$  and $W^V_{h,\mathrm{post}}\in \mathbb{R}^{r \times d_k}$. The latent cache representation is $V^* = XW^V_{h,\mathrm{pre}}$ and the absorbed query is $Q_h^* = Q_h(W^V_{h,\mathrm{post}})^\top$. For each head, the output is: 
\begin{equation}
    \mathrm{softmax}\!\left(\frac{Q_h^*(V^*)^\top}{\sqrt{r}}\right)
    V^* W^V_{h,\mathrm{post}}.
\end{equation}
Only $V^*$ needs to be stored in the cache, reducing memory from $\mathcal{O}(nd_k)$ to $\mathcal{O}(nr)$ per head. The compression and expansion projections can each independently be head-specific, group-specific, or shared across all heads.


\section{Experimental Setup}
\label{sec:setup}
In our experiments, we compare the performance of the proposed Keyless Attention method (denoted as QVV(3)) against the standard QKV
attention. We evaluate both methods across five models and four architectures. For the primary comparison, we train GPT-2
style~\citep{radford2019language} 12-layer (280M) and 36-layer (557M) decoder-only Transformers on WikiText-103~\citep{merity2017pointer}, using a 30M-token subset due to GPU memory constraints. To assess cross-architecture generalisation, we additionally evaluate on Pythia 410M~\citep{biderman2023pythia}, Qwen2 1.5B~\citep{qwen2}, and Llama 3.2 1B~\citep{llama3.2_2024}, covering three positional encoding schemes (learned absolute, partial RoPE, full RoPE), two residual designs (sequential, parallel), and two attention grouping strategies (MHA, GQA). Full architecture and implementation details are given in the supplementary material.
 
All models are implemented via Hugging Face Transformers~\citep{wolf-etal-2020-transformers} and trained from scratch with AdamW~\citep{loshchilov2019decoupled} ($\eta = 10^{-4}$, weight decay $= 0.01$), differing only in the attention mechanism. GPT-2 models use a linear learning rate schedule with 5\% warmup; Pythia 410M, Qwen2 1.5B, and Llama 3.2 1B use cosine decay with 5\% linear warmup. GPT-2 experiments report means and standard deviations over three random seeds. All experiments were conducted on a single NVIDIA A100-SXM4-80GB GPU.

We additionally evaluate zero-shot commonsense reasoning on five benchmarks using the 36-layer GPT-2 model: HellaSwag~\citep{zellers2019hellaswag}, ARC-Challenge~\citep{allenai:arc}, StoryCloze~\citep{mostafazadeh-EtAl:2016:N16-1}, SciQ~\citep{welbl-etal-2017-crowdsourcing}, and BoolQ~\citep{clark2019boolq}. All training conditions, hyperparameters, and preprocessing pipelines are held identical across models, such that the attention mechanism is the sole variable. The primary goal is to isolate the effect of Keyless Attention relative to the standard baseline. 

\section{Results}
\label{sec:results}
\subsection{Training Dynamics and Validation Performance}

Figure~\ref{fig:qkv_QVV(3)_training_ppl} and Table~\ref{tab:ablation} summarize training dynamics and validation performance across depths of the GPT-2 architecture. Both methods exhibit similar optimization behavior with low variance across seeds, but differ in overfitting and generalization as depth increases. 

The 12-layer and 36-layer models are trained in FP32 and FP16, respectively, with three runs of 10 epochs each. The per-epoch training times are 29.3/29.6 minutes (QKV/QVV(3)) for the 12-layer model and 70.2/75.5 minutes (QKV/QVV(3)) for the 36-layer model, resulting in approximately 102.30 GPU-hours of total training cost. It is worth noting that QVV(3) requires slightly more training time per epoch than QKV due to the additional optimization complexity introduced by value-space routing.

In the 12-layer setting, QKV achieves a slightly lower best validation loss (3.5180 vs.\ 3.5216) and perplexity (33.71 vs.\ 33.84), differences that are minimal relative to metric scale. QVV(3) is more robust in later epochs, with a smaller overfitting gap (+0.0519 vs.\ +0.0778) and lower final validation loss.
In the 36-layer setting, QVV(3) outperforms QKV across all metrics: lower best validation loss (3.5044 vs.\ 3.5116), lower perplexity (33.26 vs.\ 33.50), and substantially reduced overfitting (+0.6866 vs.\ +0.8404). A one-sided Welch's t-test on best validation loss shows no significant degradation for the 12-layer model (p=0.266) and a trend toward improved validation loss for the 36-layer model (p=0.070). Notably, QVV(3) consistently demonstrates lower validation loss compared to QKV after the best epoch. These results suggest that QVV(3) preserves expressivity while improving the robustness of the model.

\begin{figure*}[h]
    \centering
    \includegraphics[width=0.6\linewidth]
    {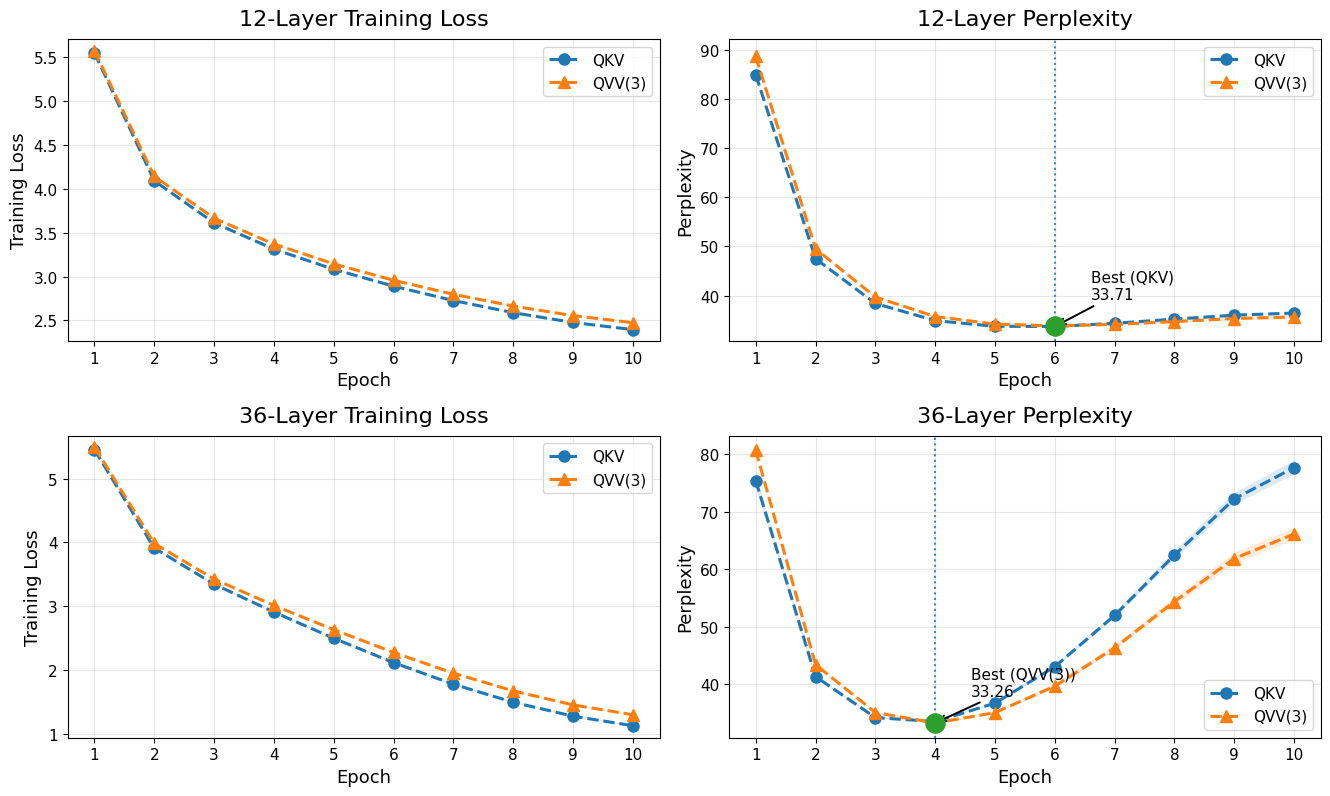}
   \caption{Training dynamics of QKV and QVV(3) across GPT-2 model
depths. Top row: 12-layer (280M parameters); bottom row: 36-layer (557M parameters). Left: training loss; right: validation perplexity. Curves are means over three seeds; shaded regions indicate one standard deviation.}
    \label{fig:qkv_QVV(3)_training_ppl}
\end{figure*}

\begin{figure*}[h]
    \centering
    \begin{subfigure}[b]{0.32\linewidth}
        \centering
        \includegraphics[width=\linewidth]{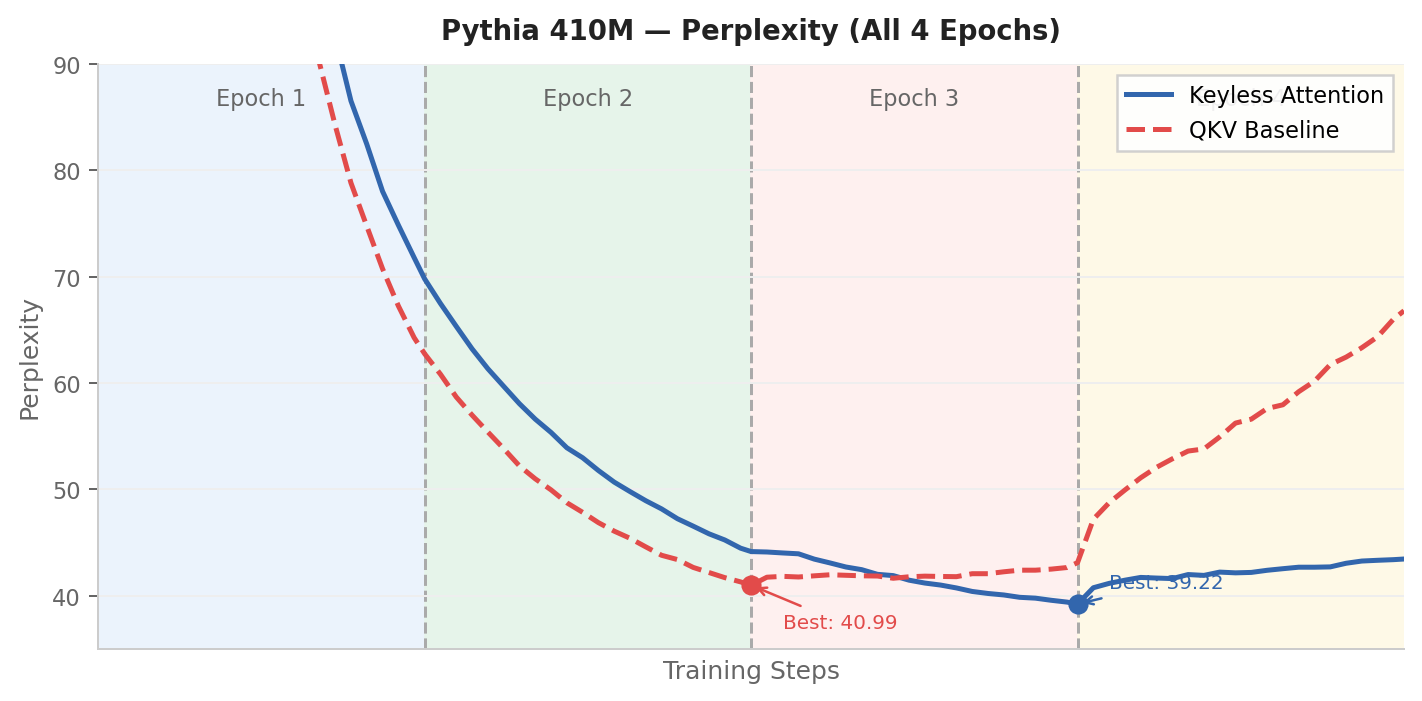}
        \caption{Pythia 410M. Keyless: 39.22 vs.\ QKV: 40.99.}
        \label{fig:pythia_ppl}
    \end{subfigure}
    \hfill
    \begin{subfigure}[b]{0.32\linewidth}
        \centering
        \includegraphics[width=\linewidth]{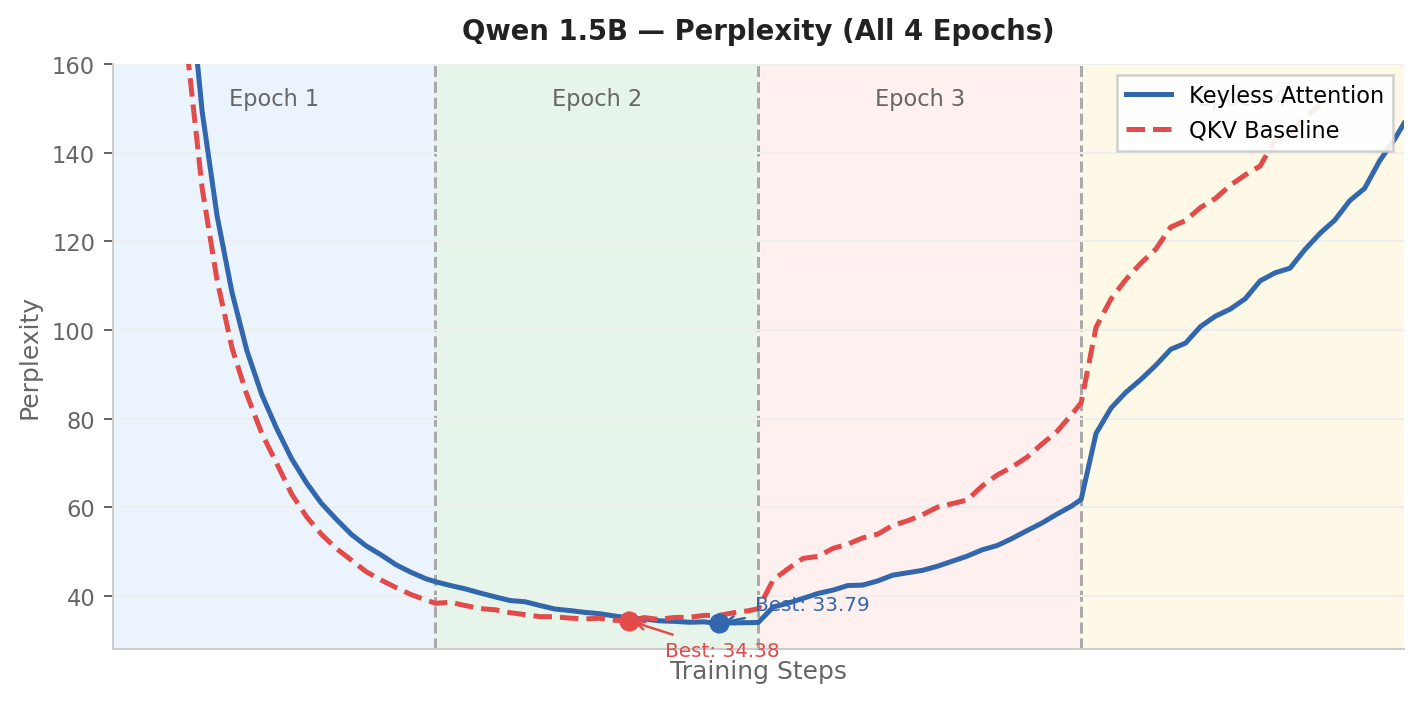}
        \caption{Qwen2 1.5B. Keyless: 33.79 vs.\ QKV: 34.38.}
        \label{fig:qwen_ppl}
    \end{subfigure}
    \hfill
    \begin{subfigure}[b]{0.32\linewidth}
        \centering
        \includegraphics[width=\linewidth]{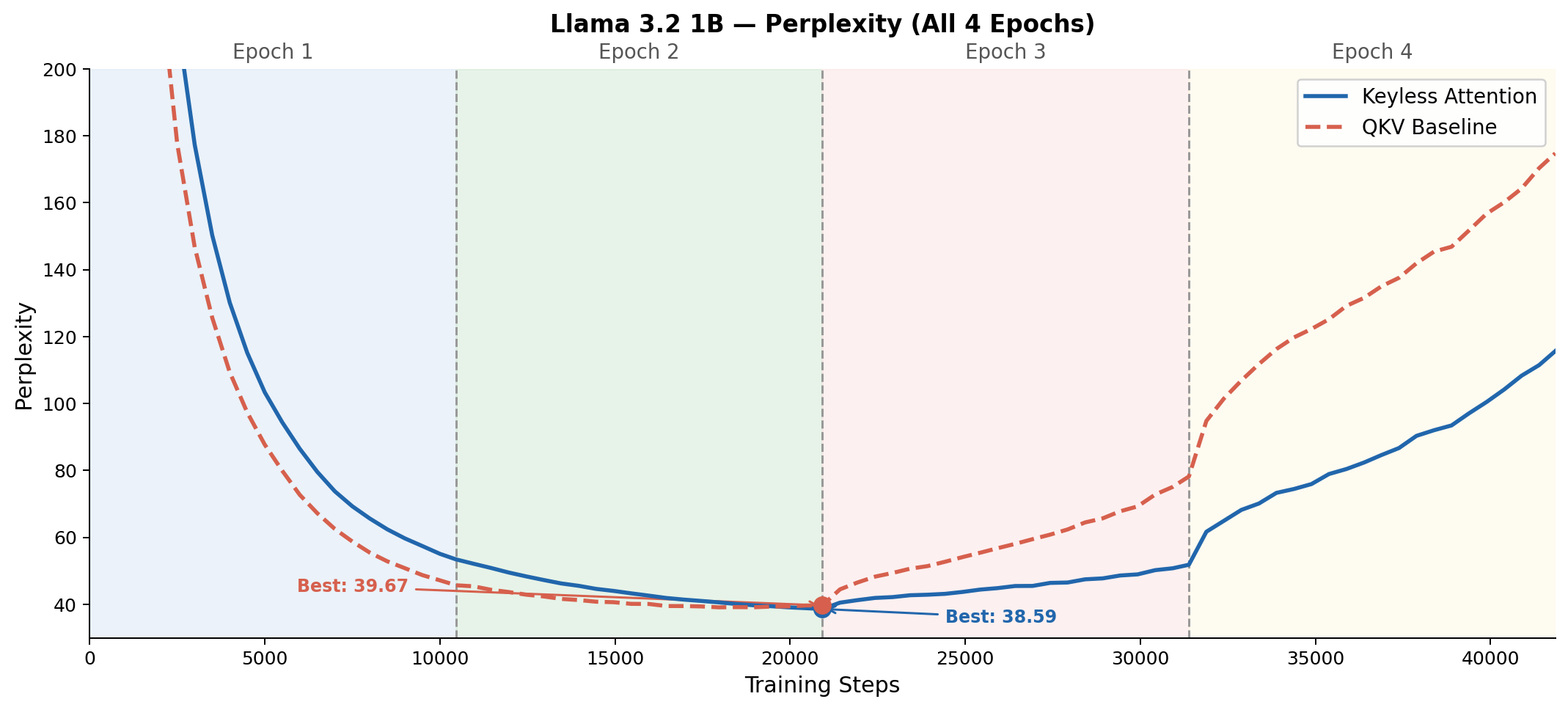}
        \caption{Llama 3.2 1B. Keyless: 38.59 vs.\ QKV: 39.09.}
        \label{fig:llama_ppl}
    \end{subfigure}
    \caption{Perplexity over 4 epochs across three GQA architectures. We observe that Keyless Attention matches or outperforms standard QKV attention at the best epoch in all three models, and degrades more slowly after the best epoch, indicating greater robustness to overfitting.}
    \label{fig:ppl_curves_combined}
\end{figure*}

\begin{figure*}[h]
\centering
\includegraphics[width=0.6\linewidth]{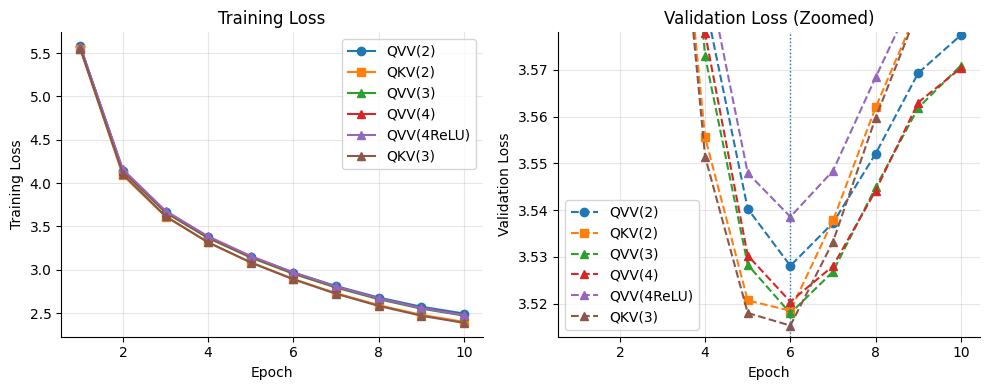}
\caption{Validation loss comparison of QVV and QKV variants across
factorization depths $m$ (12-layer GPT-2 model, WikiText-103).}
\label{fig:compare}
\end{figure*}

\begin{figure*}[t]
    \centering
    \includegraphics[width=0.6\textwidth]{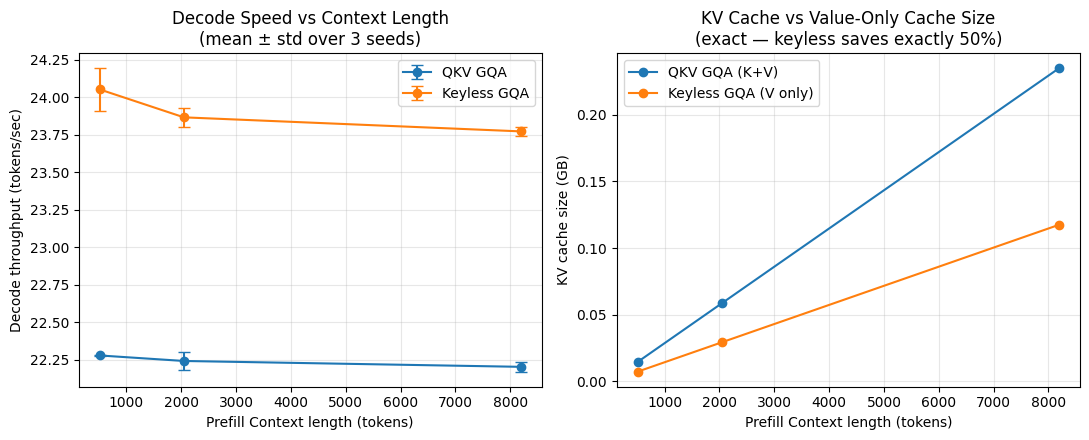}
    \caption{\textbf{Decode throughput and KV cache size for Keyless vs.\ QKV attention under GQA.} \textit{Left:} decode throughput (mean $\pm$ std over 3 seeds) as a function of context length; Keyless exceeds QKV. \textit{Right:} KV cache size vs Value-Only Cache size; Keyless reduces cache memory by exactly 50\% at every context length, since it stores only value states.}
    \label{fig:inference_efficiency}
\end{figure*}

\subsection{Downstream Task Performance}

Table~\ref{tab:combined} reports zero-shot accuracy on five benchmarks using the 36-layer GPT-2 model.
A one-sided Welch's t-test shows that QVV(3) significantly improves performance on HellaSwag (p=0.004) and StoryCloze (p=0.007), while achieving comparable performance on ARC-Challenge (p=0.475) and BoolQ (p=0.086). QKV remains stronger on SciQ (p<0.001). Overall, QVV(3) maintains competitive downstream performance while reducing KV-cache memory by 50\%.


\subsection{Ablation: Factorization Depth $m$}

Let QVV(m) denote a Keyless Attention with attention factorization depth $m$. Figure~\ref{fig:compare} compares QVV(2), QVV(3), QVV(4), QVV(4ReLU),
QKV(2), and QKV(3) on the 12-layer GPT-2 model. The method QVV(2) refers to the KV-sharing method \citep{edward2026qv, kayyam2026transformers,team2026gemma} with no value-space routing matirix $W^R.$ QKV(2) refers to the standard attention, whereas QKV(3) refers to the standard attention method with one extra projection matrix for the query representations.
QVV(4ReLU) performs worst, confirming that a nonlinear ReLU insertion
disrupts the implicit regularization of the linear factorization.
QVV(2) achieves worse validation loss than the QKV(2) baseline,
while QVV(3) and QVV(4) match QKV(2) and QKV(3) at best validation loss.
After the best epoch, QKV(2) and QKV(3) exhibit more pronounced
overfitting than QVV(3) and QVV(4). 
Since QVV(3) and QVV(4) perform comparably while QVV(4) requires one
additional weight matrix, QVV(3) offers the best efficiency--performance
trade-off for this dataset. This ablation study demonstrates that Keyless Attention with the value-routing matrix $W^R$ achieves better performance than KV-sharing method with key and value sharing the same representation.

\subsection{Comparison on Pythia, Qwen2, and Llama 3.2 Architectures}
\label{sec:pythia_qwen}

To assess cross-architecture generalization, we evaluate on Pythia 410M~\citep{biderman2023pythia}, Qwen2 1.5B~\citep{qwen2}, and Llama 3.2 1B~\citep{llama3.2_2024},
using the same dataset and optimiser as the GPT-2 experiments, with a cosine decay schedule with 5\% linear warmup. 

\paragraph{Pythia 410M.}
Pythia is a 24-layer decoder-only model (hidden size 1{,}024, 16 heads, head dim 64) with partial
RoPE~\citep{su2024roformer} and a parallel residual connection. Keyless Attention achieves a best validation loss of 3.6692 (PPL $= 39.22$) versus 3.7133 (PPL $= 40.99$) for the baseline. This is a reduction of 0.0441 in loss and 1.77 in perplexity. Notably, Keyless Attention continues to improve through epoch~3 while the baseline plateaus after epoch~2, suggesting that the gradient entanglement between $W^{R}$ and $W^V$ introduced by value-space routing slows convergence but ultimately reaches a lower minimum (Figure~\ref{fig:pythia_ppl}).

\paragraph{Qwen2 1.5B.}
Qwen2 employs GQA~\citep{ainslie2023gqa} with 12 query heads and 2 KV heads ($6\times$ grouping ratio). Our keyless implementation uses 2 value heads following the GQA format, with $W^{Q_2}$ made head-specific to preserve query-head-specific routing. Keyless Attention achieves a best validation loss of 3.5201 (PPL $= 33.79$) versus 3.5376 (PPL $= 34.38$) for the baseline, which is a reduction of 0.0175 in loss and 0.59 in perplexity.
Standard attention achieves a lower validation loss in epoch~1, but Keyless Attention closes the gap by epoch~2 and reaches a better best checkpoint. From epoch~3 onward, both models overfit while Keyless Attention maintains a consistently lower validation loss than the baseline (Figure~\ref{fig:qwen_ppl}).

\paragraph{Llama 3.2 1B.}
Llama 3.2 1B~\citep{llama3.2_2024} is a decoder-only model with 16 layers, hidden size 2{,}048, employing GQA with 32 query heads and 8 KV heads ($4\times$ grouping ratio) and full RoPE positional embeddings.
Our keyless implementation follows the same GQA adaptation as Qwen2, with value heads matching the KV head count and $W^{Q_2}$ made head-specific.
Keyless Attention achieves a best validation loss of \textbf{3.6529} (PPL $= 38.59$) versus 3.6805 (PPL $= 39.09$) for the baseline, forming a reduction of 0.0276 in loss and 0.50 in perplexity. Standard attention converges faster in epoch~1, but Keyless Attention surpasses it by the end of epoch~2 and maintains a consistently lower validation loss throughout the overfitting regime from epoch~3 onward.
(Figure~\ref{fig:llama_ppl}).

\subsection{Inference Efficiency}
\label{sec:inference_efficiency}

We benchmark decode-time throughput and KV cache memory for Keyless attention against standard QKV attention under grouped-query attention (GQA), using the Qwen2-1.5B architecture (28 layers, 2 KV heads) as a testbed. 
For Keyless Attention, we exploit the fact that its two sequential query projections, $W^{Q}$ and $W^{R}$, can be fused into a single matrix prior to inference. This eliminates any additional projection cost relative to QKV's single query projection. We measure decode throughput at prefill context lengths of $512$, $2048$, and $8192$ tokens (batch size 1, 256 generated tokens per run), averaged over 3 random seeds. As shown in Figure~\ref{fig:inference_efficiency} (left panel), Keyless Attention exceeds QKV decode throughput across all tested context lengths, consistent with Keyless eliminating the key projection ($d \times d_{kv}$) at no additional cost from the fused query projection. Figure~\ref{fig:inference_efficiency} (right panel) shows the corresponding KV cache size: since Keyless stores only value states, it reduces cache memory by exactly 50\% at every context length (e.g., $0.118$ vs.\ $0.236$ GB at 8192 tokens). We expect the speedup to grow substantially at larger batch sizes, where decoding becomes memory-bandwidth-bound and KV cache reads dominate per-step latency.

\section{Evaluation under Industrial Pretraining Conditions}
\label{sec:downstream}

\paragraph{Experimental setup.}

The experiments above evaluate Keyless
Attention across five models and four architectures on WikiText-103
(${\approx}$30M tokens), training for multiple epochs to observe
convergence behavior.
Here we conduct a complementary experiment at larger scale, training
on a 1B-token subset of FineWeb-Edu~\citep{penedo2024fineweb}, a
high-quality filtered web corpus curated for educational content.
At this scale, a single epoch already exhausts the training budget,
more closely resembling the industrial pretraining setting where
models are trained on massive datasets for one pass rather than
iterated over smaller corpora.
This experiment therefore examines the performance of Keyless
Attention under a more realistic pretraining regime.

The training set comprises approximately 980M tokens following a
98/1/1 train/validation/test split, with non-overlapping windows
of 1,024 tokens and a batch size of 2 sequences with gradient
accumulation of 32 steps (effective batch size of 65,536 tokens
per optimizer step).
Both models use the Qwen2.5-3B architecture
($d_{\text{hidden}}=2048$, 16 attention heads, 2 KV heads,
36 layers) trained with the AdamW optimizer
($\text{LR}=10^{-4}$, cosine decay to zero, 5\% warmup,
weight decay $0.01$) for one full epoch.
The QKV baseline uses the standard Qwen2ForCausalLM attention
implementation; Keyless Attention replaces the key projection
with a head-specific value-space routing projection $W^{R}$ with all other hyperparameters
identical.
At the end of training, QKV achieves a validation perplexity of
$24.83$ and Keyless Attention achieves $25.80$, a gap of $0.97$
PPL ($3.9\%$ relative).

We evaluate both models on a suite of six downstream benchmarks
spanning commonsense reasoning, reading comprehension, factual
knowledge, and long-range coherence, plus out-of-distribution
language modelling perplexity on WikiText-103~\citep{merity2017pointer}.
All multiple-choice tasks are evaluated via log-probability scoring
with no prompt tuning, following standard zero-shot evaluation
practice~\citep{brown2020language}.

\paragraph{Comparison Results.}
Table~\ref{tab:downstream} reports the results.
On out-of-distribution language modelling (WikiText-103 perplexity),
QKV outperforms Keyless Attention ($55.20$ vs $57.73$ PPL), and the in-distribution validation
perplexities at the end of training ($24.83$ vs $25.80$ PPL,
a gap of $0.97$ PPL).

On downstream accuracy tasks, the results are evenly split:
Keyless Attention outperforms QKV on HellaSwag ($0.240$ vs $0.230$),
OpenBookQA ($0.145$ vs $0.115$), and BoolQ ($0.670$ vs $0.620$),
while QKV outperforms Keyless Attention on WinoGrande
($0.525$ vs $0.475$), ARC-Easy ($0.440$ vs $0.435$), and
LAMBADA ($0.198$ vs $0.180$).
Averaging across the six accuracy tasks, Keyless Attention achieves
marginally higher mean accuracy ($0.359$ vs $0.355$).

\begin{table}[h]
\centering
\small
\begin{tabular}{lccr}
\toprule
Task & QKV & Keyless & Winner \\
\midrule
WikiText-103 PPL $\downarrow$ & 55.20 & 57.73 & QKV \\
\midrule
HellaSwag        & 0.230 & 0.240 & Keyless \\
OpenBookQA       & 0.115 & 0.145 & Keyless \\
BoolQ            & 0.620 & 0.670 & Keyless \\
WinoGrande       & 0.525 & 0.475 & QKV \\
ARC-Easy         & 0.440 & 0.435 & QKV \\
LAMBADA          & 0.198 & 0.180 & QKV \\
\midrule
Accuracy average & 0.355 & 0.359 & Keyless \\
\bottomrule
\end{tabular}
\caption{
Zero-shot downstream evaluation of QKV baseline and Keyless Attention
on Qwen2.5-3B trained from scratch on 1B tokens of FineWeb-Edu,
evaluated at final best checkpoints (full training run, matched
token count).
}
\label{tab:downstream}
\end{table}
\begin{figure*}[h]
    \centering
    \includegraphics[width=0.6\linewidth]
    {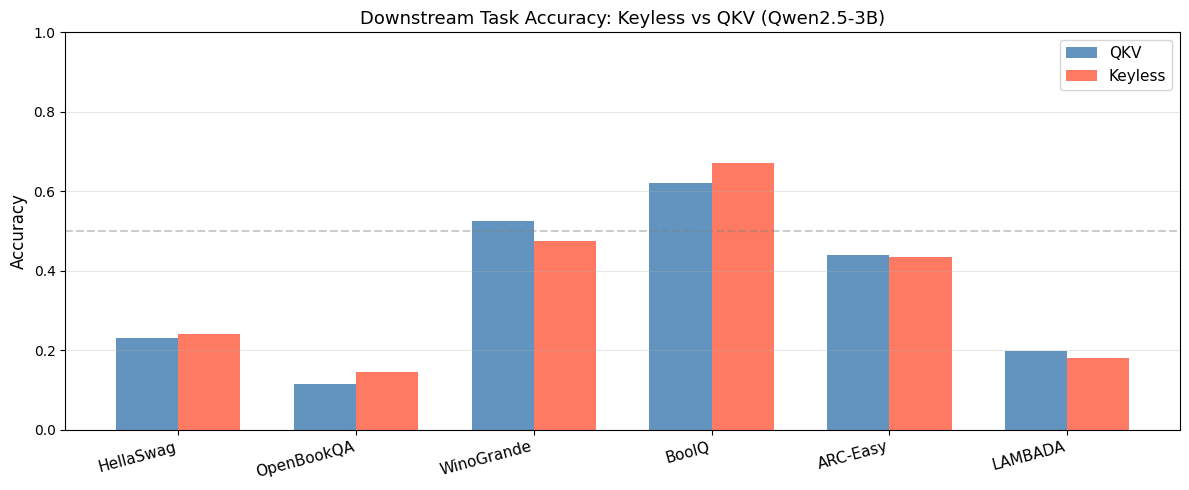}
   \caption{Zero-shot downstream task accuracy of Keyless Attention and
standard QKV attention on Qwen2.5-3B trained from scratch on 1B tokens
of FineWeb-Edu. Both models evaluated at final best checkpoints.
}
    \label{fig:1B}
\end{figure*}

\begin{figure*}[h]
    \centering
    \includegraphics[width=0.8\linewidth]
    {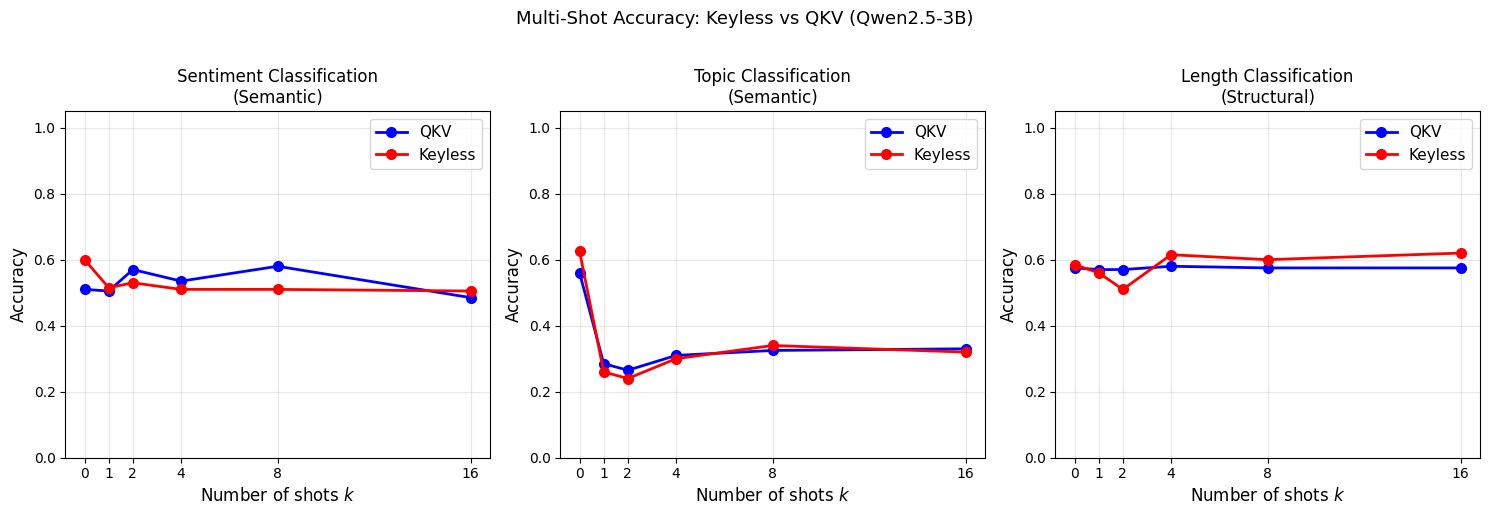}
   \caption{Multi-shot accuracy of Keyless Attention and standard QKV
attention (Qwen2.5-3B, 1B tokens of FineWeb-Edu) across $k \in
\{0, 1, 2, 4, 8, 16\}$ demonstrations on sentiment classification
(SST-2), topic classification (AGNews), and length classification.
Both models exhibit accuracy degradation from $k{=}0$ to $k{=}1$,
consistent with models trained on raw web text without
instruction-following data.}
    \label{fig:1Bmultipleshot}
\end{figure*}

\begin{table*}[h]
\centering
\small
\setlength{\tabcolsep}{4pt}
\begin{tabular}{lcccccccccccc}
\toprule
& \multicolumn{2}{c}{$k=0$}
& \multicolumn{2}{c}{$k=1$}
& \multicolumn{2}{c}{$k=2$}
& \multicolumn{2}{c}{$k=4$}
& \multicolumn{2}{c}{$k=8$}
& \multicolumn{2}{c}{$k=16$} \\
\cmidrule(lr){2-3}\cmidrule(lr){4-5}\cmidrule(lr){6-7}
\cmidrule(lr){8-9}\cmidrule(lr){10-11}\cmidrule(lr){12-13}
Task & Q & KL & Q & KL & Q & KL & Q & KL & Q & KL & Q & KL \\
\midrule
Sentiment & .510 & \textbf{.600} & .505 & \textbf{.515}
          & \textbf{.570} & .530 & \textbf{.535} & .510
          & \textbf{.580} & .510 & .485 & \textbf{.505} \\
Topic     & .560 & \textbf{.625} & \textbf{.285} & .260
          & \textbf{.265} & .240 & \textbf{.310} & .300
          & .325 & \textbf{.340} & \textbf{.330} & .320 \\
Length    & .575 & \textbf{.585} & \textbf{.570} & .560
          & \textbf{.570} & .510 & .580 & \textbf{.615}
          & .575 & \textbf{.600} & .575 & \textbf{.620} \\
\bottomrule
\end{tabular}
\caption{Multi-shot accuracy for QKV (Q) and Keyless (KL) on
sentiment, topic, and length classification.
Bold indicates the higher value per cell.
Both models show accuracy degradation from $k{=}0$ to $k{=}1$,
consistent with models not yet trained on instruction-following data.}
\label{tab:multishot}
\end{table*}

We observe that the tasks on which Keyless Attention wins are predominantly semantic in nature: HellaSwag
requires predicting plausible activity continuations based on
semantic coherence; OpenBookQA tests science knowledge grounded
in semantic associations; BoolQ requires reading comprehension
where the relevant signal is the semantic relationship between
passage and question. These tasks may benefit from the value-space routing of $W^R.$

Tasks on which QKV wins have stronger structural or syntactic components:
WinoGrande requires pronoun coreference resolution, which depends
on syntactic structure rather than semantic content;
LAMBADA requires predicting the final word of a long passage
where long-range syntactic dependencies play a large role;
These tasks may benefit
from the independent routing capacity of $W^K$.

A notable finding is that Keyless Attention achieves competitive mean
downstream accuracy ($0.359$ vs $0.355$) despite higher perplexity
($25.80$ vs $24.83$ in-distribution; $57.73$ vs $55.20$ on
WikiText-103).
This suggests that value-space routing generalizes efficiently
to semantically oriented downstream tasks than perplexity alone
would predict.


\section{Multi-Shot Evaluation}
\label{sec:multishot}

We evaluate both models on three classification tasks under varying
numbers of few-shot examples $k \in \{0, 1, 2, 4, 8, 16\}$ provided
in the prompt, using the final best checkpoints after training on
1B tokens of FineWeb-Edu. The tasks comprise two semantically oriented tasks, binary
sentiment classification (SST-2,~\citep{socher2013recursive}) and
four-class news topic classification
(AGNews,~\citep{zhang2015character}), and one structural task, 
length classification (short/long by word count), where the label
depends on token count rather than semantic content.
Multiple-choice tasks are scored via next-token logit comparison
at the label position.

\paragraph{Results.}
Table~\ref{tab:multishot} and Figure~\ref{fig:1Bmultipleshot} report
the accuracy across all values of $k$.
A consistent pattern across all three tasks is a sharp drop from
$k{=}0$ to $k{=}1$, followed by partial recovery at larger $k$.
This behavior is characteristic of models trained on raw web text
without instruction-following data: the model has not learned the
meta-skill of pattern-matching from demonstrations, so additional
context initially disrupts rather than guides prediction.

At zero-shot ($k{=}0$), Keyless Attention outperforms QKV on all
three tasks: sentiment ($0.600$ vs $0.510$), topic ($0.625$ vs
$0.560$), and length ($0.585$ vs $0.575$). Beyond zero-shot, the results are noisy and no consistent trend
emerges.
On the sentiment task, QKV recovers better at $k{=}8$ ($0.580$
vs $0.510$), while Keyless leads at $k{=}0$ and $k{=}16$.
On the topic task, Keyless leads at $k{=}0$ and $k{=}8$, while
QKV is comparable or better at other shot counts.
Notably, on the structural length task, Keyless Attention improves
consistently from $k{=}0$ ($0.585$) to $k{=}16$ ($0.620$), while
QKV remains essentially flat ($0.575$--$0.580$ throughout).

The multi-shot results are less informative at this training stage:
the sharp k=0 to k=1 degradation suggests neither model has
developed reliable in-context learning capability after 1B tokens
of raw web text pretraining, making multi-shot comparisons noisy
and difficult to interpret theoretically. We note that robust multi-shot evaluation requires either instruction tuning or substantially more pretraining compute, neither of which
is available at the current training scale.

Across both zero-shot downstream benchmarks and multi-shot
evaluations, Keyless Attention achieves comparable task performance
to standard QKV attention, with neither method consistently
dominating.
The modest perplexity gap ($3.9\%$ relative) reflects the
convergence lag from gradient entanglement rather than a
fundamental quality difference.
These results support the viability of Keyless Attention in
industrial pretraining settings: the proposed method delivers
performance on par with standard attention while permanently
eliminating the key cache.

\section{Conclusion}

The motivation of value-space routing in Keyless Attention is to replace conventional key-space routing and eliminate key representations from the attention computation. As a result, our proposed method reduces KV-cache memory consumption by 50\%, while preserving its performance. Experiments across various models and architectures demonstrate
that Keyless Attention achieves comparable perplexity and downstream
task performance relative to standard attention, with neither method
consistently dominating. Furthermore, Keyless Attention exhibits slower degradation in validation loss after the best epoch compared with standard attention, indicating improved robustness against overfitting. Additional experiments show that Keyless Attention with dedicated value-space routing outperforms KV-sharing methods, highlighting the importance of incorporating a value-space routing projection into the attention computation. Experiments in the pretraining regime confirm the viability of
Keyless Attention in industrial settings. Although Keyless Attention requires slightly more computation during training, it provides substantial inference-time benefits through lower memory consumption and improved decoding efficiency. We believe Keyless Attention provides a promising new direction for designing memory-efficient Transformer architectures for large-scale language models. Future work will investigate the scaling behavior of Keyless Attention on larger datasets, longer context windows, and larger language models.



\newpage

\bibliographystyle{apalike}
\bibliography{custom}

@inproceedings{vaswani2017attention,
  author={Ashish Vaswani and Noam Shazeer and Niki Parmar and Jakob Uszkoreit and Llion Jones and Aidan N. Gomez and Lukasz Kaiser and Illia Polosukhin},
  title     = {Attention Is All You Need},
  booktitle = {Advances in Neural Information Processing Systems},
  year      = {2017}
}

@article{beltagy2020longformer,
  author  = {Beltagy, Iz and Peters, Matthew E. and Cohan, Arman},
  title   = {Longformer: The Long-Document Transformer},
  journal = {arXiv preprint arXiv:2004.05150},
  year    = {2020}
}

@article{child2019sparse,
  title={Generating Long Sequences with Sparse Transformers},
   author={Rewon Child and Scott Gray and Alec Radford and Ilya Sutskever},
  journal={arXiv preprint arXiv:1904.10509},
  year={2019}
}

@inproceedings{katharopoulos2020transformers,
  title     = {Transformers are {RNN}s: Fast Autoregressive Transformers with Linear Attention},
  author    = {Katharopoulos, Angelos and Vyas, Apoorv and Pappas, Nikolaos and Fleuret, Fran{\c{c}}ois},
  booktitle = {Proceedings of the 37th International Conference on Machine Learning (ICML)},
  pages     = {5156--5165},
  year      = {2020},
  volume    = {119},
  series    = {Proceedings of Machine Learning Research},
  publisher = {PMLR},
  url       = {https://proceedings.mlr.press/v119/katharopoulos20a.html}
}

@inproceedings{choromanski2021rethinking,
  title     = {Rethinking Attention with Performers},
  author    = {Choromanski, Krzysztof and Likhosherstov, Valerii and Dohan, David and Song, Xingyou and Gane, Andreea and Sarlos, Tamas and Hawkins, Peter and Davis, Jared and Mohiuddin, Afroz and Kaiser, Lukasz and Belanger, David and Colwell, Lucy and Weller, Adrian},
  booktitle = {International Conference on Learning Representations (ICLR)},
  year      = {2021},
  url       = {https://openreview.net/forum?id=Ua6zuk0WRH}
}

@article{wang2020linformer,
  title   = {Linformer: Self-Attention with Linear Complexity},
  author  = {Wang, Sinong and Li, Belinda Z. and Khabsa, Madian and Fang, Han and Ma, Hao},
  journal = {arXiv preprint arXiv:2006.04768},
  year    = {2020},
  url     = {https://arxiv.org/abs/2006.04768}
}

@inproceedings{brown2020language,
  title     = {Language Models are Few-Shot Learners},
  author    = {Brown, Tom B. and Mann, Benjamin and Ryder, Nick and Subbiah, Melanie and Kaplan, Jared and Dhariwal, Prafulla and Neelakantan, Arvind and Shyam, Pranav and Sastry, Girish and Askell, Amanda and Agarwal, Sandhini and Herbert-Voss, Ariel and Krueger, Gretchen and Henighan, Tom and Child, Rewon and Ramesh, Aditya and Ziegler, Daniel M. and Wu, Jeffrey and Winter, Clemens and Hesse, Christopher and Chen, Mark and Sigler, Eric and Litwin, Mateusz and Gray, Scott and Chess, Benjamin and Clark, Jack and Berner, Christopher and McCandlish, Sam and Radford, Alec and Sutskever, Ilya and Amodei, Dario},
  booktitle = {Advances in Neural Information Processing Systems (NeurIPS)},
  volume    = {33},
  pages     = {1877--1901},
  year      = {2020}
}

@inproceedings{kwon2023efficient,
  title     = {Efficient Memory Management for Large Language Model Serving with {PagedAttention}},
  author    = {Kwon, Woosuk and Li, Zhuohan and Zhuang, Siyuan and Sheng, Ying and Zheng, Lianmin and Yu, Cody Hao and Gonzalez, Joseph E. and Zhang, Hao and Stoica, Ion},
  booktitle = {Proceedings of the 29th Symposium on Operating Systems Principles (SOSP)},
  year      = {2023},
  url       = {https://arxiv.org/abs/2309.06180}
}

@inproceedings{prabhu2025vattention,
  title={vattention: Dynamic memory management for serving llms without pagedattention},
  author={Prabhu, Ramya and Nayak, Ajay and Mohan, Jayashree and Ramjee, Ramachandran and Panwar, Ashish},
  booktitle={Proceedings of the 30th ACM International Conference on Architectural Support for Programming Languages and Operating Systems, Volume 1},
  pages={1133--1150},
  year={2025}
}

@inproceedings{liao2026zipage,
  title={Zipage: Maintain High Request Concurrency for LLM Reasoning through Compressed PagedAttention},
  author={Liao, Mengqi and Wang, Lu and Zhang, Chaoyun and Qiao, Bo and Qin, Si and Lin, Qingwei and Rajmohan, Saravan and Zhang, Dongmei and Wan, Huaiyu},
  booktitle={Findings of the Association for Computational Linguistics: ACL 2026},
  pages={7716--7737},
  year={2026}
}

@article{kang2025turboattention,
  title={TurboAttention: Efficient attention approximation for high throughputs llm},
  author={Kang, Hao and Bharadwaj, Srikant and Hensman, James and Krishna, Tushar and R{\"u}hle, Victor and Rajmohan, Saravan},
  journal={Proceedings of Machine Learning and Systems},
  volume={7},
  year={2025}
}

@article{shazeer2019fast,
  title         = {Fast Transformer Decoding: One Write-Head is All You Need},
  author        = {Shazeer, Noam},
  journal = {arXiv preprint arXiv:1911.02150},
  year          = {2019},
  eprint        = {1911.02150},
  archivePrefix = {arXiv},
  primaryClass  = {cs.NE},
  url           = {https://arxiv.org/abs/1911.02150}
}

@inproceedings{ainslie2023gqa,
  title     = {{GQA}: Training Generalized Multi-Query Transformer Models from Multi-Head Checkpoints},
  author    = {Ainslie, Joshua and Lee-Thorp, James and de Jong, Michiel and Zemlyanskiy, Yury and Lebr{\'o}n, Federico and Sanghai, Sumit},
  booktitle = {Proceedings of the 2023 Conference on Empirical Methods in Natural Language Processing (EMNLP)},
  pages     = {4895--4901},
  year      = {2023},
  address   = {Singapore},
  publisher = {Association for Computational Linguistics},
  url       = {https://aclanthology.org/2023.emnlp-main.298/}
}

@inproceedings{tang2025razorattention,
  title={Razorattention: Efficient kv cache compression through retrieval heads},
  author={Tang, Hanlin and Lin, Yang and Lin, Jing and Han, Qingsen and Ke, Danning and Hong, Shikuan and Yao, Yiwu and Wang, Gongyi},
  booktitle={International Conference on Learning Representations},
  volume={2025},
  pages={16632--16646},
  year={2025}
}

@article{team2026gemma,
  title={Gemma 4 technical report},
  author={Team, Gemma and Abd, Sherif El and Aggarwal, Vaibhav and Algayres, Robin and Andreev, Alek and Bachem, Olivier and Ballantyne, Ian and Brick, Cormac and C{\u{a}}rbune, Victor and Casbon, Michelle and others},
  journal={arXiv preprint arXiv:2607.02770},
  year={2026}
}

@inproceedings{wen2025token,
  title={Token pruning in multimodal large language models: Are we solving the right problem?},
  author={Wen, Zichen and Gao, Yifeng and Li, Weijia and He, Conghui and Zhang, Linfeng},
  booktitle={Findings of the Association for Computational Linguistics: ACL 2025},
  pages={15537--15549},
  year={2025}
}

@inproceedings{xu2025think,
  title={Think: Thinner key cache by query-driven pruning},
  author={Xu, Yuhui and Jie, Zhanming and Dong, Hanze and Wang, Lei and Lu, Xudong and Zhou, Aojun and Saha, Amrita and Xiong, Caiming and Sahoo, Doyen},
  booktitle={International Conference on Learning Representations},
  volume={2025},
  pages={56691--56709},
  year={2025}
}

@article{huang2025kv,
  title={KV Admission: Learning What to Write for Efficient Long-Context Inference},
  author={Huang, Yen-Chieh and Hsiu, Pi-Cheng and Fang, Rui and Chen, Ming-Syan},
  journal={arXiv preprint arXiv:2512.17452},
  year={2025}
}

@inproceedings{pope2023efficiently,
  title     = {Efficiently Scaling Transformer Inference},
  author    = {Pope, Reiner and Douglas, Sholto and Chowdhery, Aakanksha
               and Devlin, Jacob and Bradbury, James and Heek, Jonathan
               and Xiao, Kefan and Agrawal, Shivani and Dean, Jeff},
  booktitle = {Proceedings of Machine Learning and Systems},
  volume    = {5},
  year      = {2023}
}

@article{chowdhery2022palmscalinglanguagemodeling,
      title={PaLM: Scaling Language Modeling with Pathways}, 
      author={Aakanksha Chowdhery and Sharan Narang and Jacob Devlin and Maarten Bosma and Gaurav Mishra and Adam Roberts and Paul Barham and Hyung Won Chung and Charles Sutton and Sebastian Gehrmann and Parker Schuh and Kensen Shi and Sasha Tsvyashchenko and Joshua Maynez and Abhishek Rao and Parker Barnes and Yi Tay and Noam Shazeer and Vinodkumar Prabhakaran and Emily Reif and Nan Du and Ben Hutchinson and Reiner Pope and James Bradbury and Jacob Austin and Michael Isard and Guy Gur-Ari and Pengcheng Yin and Toju Duke and Anselm Levskaya and Sanjay Ghemawat and Sunipa Dev and Henryk Michalewski and Xavier Garcia and Vedant Misra and Kevin Robinson and Liam Fedus and Denny Zhou and Daphne Ippolito and David Luan and Hyeontaek Lim and Barret Zoph and Alexander Spiridonov and Ryan Sepassi and David Dohan and Shivani Agrawal and Mark Omernick and Andrew M. Dai and Thanumalayan Sankaranarayana Pillai and Marie Pellat and Aitor Lewkowycz and Erica Moreira and Rewon Child and Oleksandr Polozov and Katherine Lee and Zongwei Zhou and Xuezhi Wang and Brennan Saeta and Mark Diaz and Orhan Firat and Michele Catasta and Jason Wei and Kathy Meier-Hellstern and Douglas Eck and Jeff Dean and Slav Petrov and Noah Fiedel},
      year={2022},
      journal = {arXiv preprint arXiv:2204.02311},
      eprint={2204.02311},
      archivePrefix={arXiv},
      primaryClass={cs.CL},
      url={https://arxiv.org/abs/2204.02311}, 
}

@article{deepseekv2,
  title   = {{DeepSeek-V2}: A Strong, Economical, and Efficient
             Mixture-of-Experts Language Model},
  author  = {DeepSeek-AI},
  journal = {arXiv preprint arXiv:2405.04434},
  year    = {2024}
}

@inproceedings{luong2015effective,
  title     = {Effective Approaches to Attention-based Neural Machine Translation},
  author    = {Luong, Minh-Thang and Pham, Hieu and Manning, Christopher D.},
  booktitle = {Proceedings of the 2015 Conference on Empirical Methods in Natural Language Processing},
  pages     = {1412--1421},
  year      = {2015},
  publisher = {Association for Computational Linguistics},
  address   = {Lisbon, Portugal}
}

@inproceedings{zellers2019hellaswag,
  title     = {HellaSwag: Can a Machine Really Finish Your Sentence?},
  author    = {Zellers, Rowan and Holtzman, Ari and Bisk, Yonatan and Farhadi, Ali and Choi, Yejin},
  booktitle = {Proceedings of the 57th Annual Meeting of the Association for Computational Linguistics},
  year      = {2019}
}

@article{allenai:arc,
  author  = {Peter Clark and Isaac Cowhey and Oren Etzioni and Tushar Khot and
             Ashish Sabharwal and Carissa Schoenick and Oyvind Tafjord},
  title   = {Think you have Solved Question Answering? Try ARC, the AI2 Reasoning Challenge},
  journal = {arXiv:1803.05457v1},
  year    = {2018}
}

@inproceedings{mostafazadeh-EtAl:2016:N16-1,
  author    = {Mostafazadeh, Nasrin and Chambers, Nathanael and He, Xiaodong and
               Parikh, Devi and Batra, Dhruv and Vanderwende, Lucy and Kohli, Pushmeet and Allen, James},
  title     = {A Corpus and Cloze Evaluation for Deeper Understanding of Commonsense Stories},
  booktitle = {Proceedings of the 2016 Conference of the North American Chapter of the
               Association for Computational Linguistics: Human Language Technologies},
  month     = {June},
  year      = {2016},
  address   = {San Diego, California},
  publisher = {Association for Computational Linguistics},
  pages     = {839--849},
  url       = {http://www.aclweb.org/anthology/N16-1098}
}

@inproceedings{welbl-etal-2017-crowdsourcing,
  title     = {Crowdsourcing Multiple Choice Science Questions},
  author    = {Welbl, Johannes and Liu, Nelson F. and Gardner, Matt},
  booktitle = {Proceedings of the 3rd Workshop on Noisy User-generated Text},
  year      = {2017},
  pages     = {94--106},
  publisher = {Association for Computational Linguistics}
}

@inproceedings{clark2019boolq,
  title     = {BoolQ: Exploring the Surprising Difficulty of Natural Yes/No Questions},
  author    = {Clark, Christopher and Lee, Kenton and Chang, Ming-Wei and
               Kwiatkowski, Tom and Collins, Michael and Toutanova, Kristina},
  booktitle = {NAACL},
  year      = {2019}
}

@article{radford2019language,
  title   = {Language Models are Unsupervised Multitask Learners},
  author  = {Radford, Alec and Wu, Jeff and Child, Rewon and Luan, David and Amodei, Dario and Sutskever, Ilya},
  journal = {OpenAI Blog},
  volume  = {1},
  number  = {8},
  pages   = {9},
  year    = {2019}
}

@inproceedings{wolf-etal-2020-transformers,
  title     = {Transformers: State-of-the-Art Natural Language Processing},
  author    = {Wolf, Thomas and Debut, Lysandre and Sanh, Victor and Chaumond, Julien and
               Delangue, Clement and Moi, Anthony and Cistac, Pierric and Rault, Tim and
               Louf, R{\'e}mi and Funtowicz, Morgan and Davison, Joe and Shleifer, Sam and
               von Platen, Patrick and Ma, Clara and Jernite, Yacine and Plu, Julien and
               Xu, Canwen and Le Scao, Teven and Gugger, Sylvain and Drame, Mariama and
               Lhoest, Quentin and Rush, Alexander M.},
  booktitle = {Proceedings of the 2020 Conference on Empirical Methods in Natural Language Processing: System Demonstrations},
  year      = {2020},
  pages     = {38--45},
  publisher = {Association for Computational Linguistics},
  doi       = {10.18653/v1/2020.emnlp-demos.6}
}

@inproceedings{loshchilov2019decoupled,
  title     = {Decoupled Weight Decay Regularization},
  author    = {Loshchilov, Ilya and Hutter, Frank},
  booktitle = {International Conference on Learning Representations},
  year      = {2019},
  url       = {https://openreview.net/forum?id=Bkg6RiCqY7}
}

@inproceedings{merity2017pointer,
  title     = {Pointer Sentinel Mixture Models},
  author    = {Merity, Stephen and Xiong, Caiming and Bradbury, James and Socher, Richard},
  booktitle = {International Conference on Learning Representations},
  year      = {2017},
  url       = {https://openreview.net/forum?id=Byj72udxe}
}

@inproceedings{biderman2023pythia,
  title     = {Pythia: A Suite for Analyzing Large Language Models Across Training and Scaling},
  author    = {Biderman, Stella and Schoelkopf, Hailey and Anthony, Quentin and
               Bradley, Herbie and O'Brien, Kyle and Hallahan, Eric and
               Khan, Mohammad Aflah and Purohit, Shivanshu and
               Prashanth, USVSN Sai and Raff, Edward and Skowron, Aviya and
               Sutawika, Lintang and van der Wal, Oskar},
  booktitle = {Proceedings of the 40th International Conference on Machine Learning},
  series    = {Proceedings of Machine Learning Research},
  volume    = {202},
  pages     = {2397--2430},
  publisher = {PMLR},
  year      = {2023},
  url       = {https://arxiv.org/abs/2304.01373}
}

@article{qwen2,
  title   = {Qwen2 Technical Report},
  author  = {Yang, An and Yang, Baosong and Hui, Binyuan and Zheng, Bo and
             Yu, Bowen and Zhou, Chang and Li, Chengpeng and Li, Chengyuan and
             Liu, Dayiheng and Huang, Fei and Dong, Guanting and Wei, Haoran and
             Lin, Huan and Tang, Jialong and Wang, Jialin and Yang, Jian and
             Tu, Jianhong and Zhang, Jianwei and Ma, Jianxin and Xu, Jin and
             Zhou, Jingren and Bai, Jinze and others},
  journal = {arXiv preprint arXiv:2407.10671},
  year    = {2024},
  url     = {https://arxiv.org/abs/2407.10671}
}

@article{su2024roformer,
  title   = {{RoFormer}: Enhanced Transformer with Rotary Position Embedding},
  author  = {Su, Jianlin and Ahmed, Murtadha and Lu, Yu and Pan, Shengfeng and
             Bo, Wen and Liu, Yunfeng},
  journal = {Neurocomputing},
  volume  = {568},
  pages   = {127063},
  year    = {2024},
  doi     = {10.1016/j.neucom.2023.127063},
  url     = {https://arxiv.org/abs/2104.09864}
}

@misc{llama3.2_2024,
  title        = {Llama 3.2: Revolutionizing Edge {AI} and Vision
                  with Open, Customizable Models},
  author       = {{AI at Meta}},
  year         = {2024},
  month        = sep,
  howpublished = {Meta AI Blog},
  url          = {https://ai.meta.com/blog/llama-3-2-connect-2024-vision-edge-mobile-devices/}
}

@article{slim_attention2025,
  title={Slim attention: cut your context memory in half without loss of accuracy -- K-cache is all you need for MHA},
  author={Graef, Nils and Wasielewski, Andrew},
  journal={arXiv preprint arXiv:2503.05840},
  year={2025},
  eprint={2503.05840},
  archivePrefix={arXiv},
  primaryClass={cs.LG},
  url={https://arxiv.org/abs/2503.05840}
}

@article{edward2026qv,
  title={QV May Be Enough: Toward the Essence of Attention in LLMs},
  author={Edward, Zhang},
  journal={arXiv preprint arXiv:2603.15665},
  year={2026}
}

@inproceedings{kayyam2026transformers,
  title={Do Transformers Need Three Projections? Systematic Study of QKV Variants},
  author={Kayyam, Ali and Gopal, Anusha Madan and Lewis, M Anthony},
  booktitle={Forty-third International Conference on Machine Learning},
  year={2026}
}

@article{penedo2024fineweb,
  title={The fineweb datasets: Decanting the web for the finest text data at scale},
  author={Penedo, Guilherme and Kydl{\'\i}{\v{c}}ek, Hynek and Lozhkov, Anton and Mitchell, Margaret and Raffel, Colin and Von Werra, Leandro and Wolf, Thomas and others},
  journal={Advances in Neural Information Processing Systems},
  volume={37},
  pages={30811--30849},
  year={2024}
}

@inproceedings{socher2013recursive,
  title={Recursive deep models for semantic compositionality
         over a sentiment treebank},
  author={Socher, Richard and Perelygin, Alex and Wu, Jean
          and Chuang, Jason and Manning, Christopher D
          and Ng, Andrew Y and Potts, Christopher},
  booktitle={Proceedings of the 2013 Conference on Empirical
             Methods in Natural Language Processing},
  pages={1631--1642},
  year={2013}
}

@inproceedings{zhang2015character,
  title={Character-level convolutional networks for text classification},
  author={Zhang, Xiang and Zhao, Junbo and LeCun, Yann},
  booktitle={Advances in Neural Information Processing Systems},
  volume={28},
  year={2015}
}

\appendix
\section{Proofs of Theorems}
\label{appendix:proof_theorem}



\subsection{Theoretical Equivalence of Keyless and Standard Attention}
\label{sec:equivalence}

\noindent Proof of Theorem \ref{thm:single_head}
\begin{proof}
Since $W^V \in \mathbb{R}^{d \times d}$ has full rank $d$,
the matrix $(W^V)^\top$ is invertible.
Define $\tilde{W}^Q = W^Q(W^K)^\top(W^V)^{-\top}$.
Then:
\begin{align}
\begin{split}
    \tilde{W}^Q(W^V)^\top
    &= W^Q(W^K)^\top (W^V)^{-\top}(W^V)^\top\\
    &= W^Q(W^K)^\top \cdot I_d
    = W^Q(W^K)^\top.
\end{split}
\end{align}
Uniqueness follows from the invertibility of $(W^V)^\top$:
if $\tilde{W}^Q_1(W^V)^\top = \tilde{W}^Q_2(W^V)^\top$,
then $(\tilde{W}^Q_1 - \tilde{W}^Q_2)(W^V)^\top = 0$,
and right-multiplying by $(W^V)^{-\top}$ gives
$\tilde{W}^Q_1 = \tilde{W}^Q_2$. $\qed$
\end{proof}

\begin{remark}
\label{rem:single_head}
In the single-head case with $d_k = d$, the value projection
$W^V \in \mathbb{R}^{d \times d}$ is square and generically
invertible under standard random initialization.
The result is therefore unconditional: for any $W^Q$ and $W^K$,
a unique equivalent Keyless Attention parametrization always exists.
The multi-head case below requires an additional condition because
$W^V_h \in \mathbb{R}^{d \times d_k}$ is rectangular with $d_k < d$
and is no longer invertible.
\end{remark}


\noindent Proof of Theorem \ref{thm:multi_head}

\begin{proof}
Fix head $h$.
Let $B_h = (W^V_h)^\top \in \mathbb{R}^{d_k \times d}$.
Since $W^V_h$ has full column rank $d_k$,
the Gram matrix $(W^V_h)^\top W^V_h$ is invertible and
$B_h$ admits a right pseudoinverse
$B_h^+ = W^V_h((W^V_h)^\top W^V_h)^{-1} \in \mathbb{R}^{d \times d_k}.$
 
Setting $\tilde{W}^Q_h = \Omega_h B_h^+$ and expanding:
\begin{equation}
    \tilde{W}^Q_h(W^V_h)^\top
    = \Omega_h B_h^+ B_h
    = \Omega_h P_{W^V_h},
\end{equation}
where $$P_{W^V_h} = B_h^+ B_h = W^V_h((W^V_h)^\top W^V_h)^{-1}
(W^V_h)^\top$$ is the orthogonal projector onto $\mathrm{col}(W^V_h)$.
By condition~\eqref{eq:existence_condition},
$\mathrm{col}(\Omega_h^\top) \subseteq \mathrm{col}(W^V_h)$,
so $\Omega_h P_{W^V_h} = \Omega_h$, giving
$\tilde{W}^Q_h(W^V_h)^\top = \Omega_h = W^Q_h(W^K_h)^\top$.
 
For the general solution: any $\tilde{W}^Q_h$ satisfying
$\tilde{W}^Q_h B_h = \Omega_h$ can be written as
$\tilde{W}^Q_h = \Omega_h B_h^+ + N$ where $NB_h = 0$,
i.e.\ $N(W^V_h)^\top = 0$.
The proof is independent for each head $h$;
no relationship between heads is assumed. 
\end{proof}



\section{Architecture and Implementation Details}
\label{appendix:arch}

\subsection{Pythia 410M}
Pythia 410M uses 24 layers, hidden size 1,024, 16 attention heads
(head dim 64), intermediate size 4,096, GELU activations, partial
RoPE ($\theta = 10{,}000$, factor $= 0.25$), and a parallel residual
connection. Both models are initialised from scratch using the official
\texttt{GPTNeoXConfig}. Training uses AdamW with $\eta = 10^{-4}$,
weight decay 0.01, cosine decay with 5\% linear warmup, gradient
accumulation 2, and early stopping with patience 2 on validation loss.

\subsection{Qwen2 1.5B}
Qwen2 1.5B uses 28 layers, hidden size 1,536, 12 query heads and 2 KV
heads (GQA, $6\times$ grouping, head dim 128), and full RoPE.
In our Keyless Attention implementation, we use 2 value-routing heads
following the GQA format, with $W^{Q_2}$ made head-specific (one matrix
per query head) to preserve query-head-specific routing.
The causal mask is constructed on-the-fly to avoid pre-registering
large buffers at sequence length 32,768.
Training hyperparameters are identical to Pythia 410M above.

\subsection{Llama 3.2 1B}
Llama 3.2 1B~\citep{llama3.2_2024} uses 16 layers, hidden size 2,048,
32 query heads and 8 KV heads (GQA, $4\times$ grouping, head dim 64),
intermediate size 8,192, SiLU activations, and full RoPE with the
Llama3-scaled variant ($\theta = 500{,}000$, original maximum position
length 8,192, scaled to 131,072 via high-frequency and low-frequency
factors).
Input and output embeddings are tied (\texttt{tie\_word\_embeddings =
true}). In our Keyless Attention implementation, we follow the same GQA
adaptation as Qwen2: we use 8 value-routing heads matching the KV head
count, with $W^{Q_2}$ made head-specific (one matrix per query head)
to preserve query-head-specific routing across the $4\times$ grouping.
Training hyperparameters are identical to Pythia 410M and Qwen2 1.5B.

\begin{table}[h]
\small
\centering
\caption{Best validation loss and perplexity across all five models.
$\dagger$~Qwen2 1.5B and Llama 3.2 1B use GQA;
cache reduction applies to the value cache.}
\label{tab:all_models}
\begin{tabular}{llccc}
\toprule
\textbf{Model} & \textbf{Method}
  & \textbf{Val Loss} & \textbf{PPL} & \textbf{Cache} \\
\midrule
GPT-2 280M
  & QKV     & 3.5180 & 33.71 & 100\% \\
  & Keyless & 3.5216 & 33.84 & \textbf{50\%} \\
\midrule
GPT-2 557M
  & QKV     & 3.5116 & 33.50 & 100\% \\
  & Keyless & \textbf{3.5044} & \textbf{33.26} & \textbf{50\%} \\
\midrule
Pythia 410M
  & QKV     & 3.7133 & 40.99 & 100\% \\
  & Keyless & \textbf{3.6692} & \textbf{39.22} & \textbf{50\%} \\
\midrule
Qwen2 1.5B$^\dagger$
  & QKV     & 3.5376 & 34.38 & 100\% \\
  & Keyless & \textbf{3.5201} & \textbf{33.79} & \textbf{50\%} \\
\midrule
Llama 3.2 1B$^\dagger$
  & QKV     & 3.6805 & 39.09 & 100\% \\
  & Keyless & \textbf{3.6529} & \textbf{38.59} & \textbf{50\%} \\
\bottomrule
\end{tabular}
\end{table}

\subsection{Cross-Architecture Summary}
\label{sec:cross_arch}

Results across all five models are summarized in
Table~\ref{tab:all_models}. The five models span three distinct positional encoding schemes,
two residual connection designs, and two attention grouping
strategies, providing a broad cross-architecture validation of
Keyless Attention. Keyless Attention wins on perplexity in 4 out of 5 models and on
4 out of 5 downstream benchmarks, with the 50\% KV cache reduction
holding exactly across all architectures regardless of head count
or grouping ratio.
Beyond the best-epoch results, the post-peak validation trajectories 
reveal a consistent pattern:
across all five models, Keyless Attention's validation loss degrades
more slowly after the best epoch than standard attention's, indicating
greater robustness to overfitting. This pattern holds regardless of
architecture, head configuration, or model scale.
The consistent improvements are particularly compelling given the
breadth of architectural variation summarized in
Table~\ref{tab:arch_comparison}, underscoring that the gains are
not specific to any single design choice.

\begin{table}[h]
\small
\centering
\caption{Architectural properties of the five models evaluated.
pRoPE: partial RoPE.
GQA: Grouped Query Attention.
MHA: Multi-Head Attention.
Par: parallel residual; Seq: sequential residual.}
\label{tab:arch_comparison}
\begin{tabular}{lccccc}
\toprule
\textbf{Property}
  & \textbf{GPT-2}
  & \textbf{GPT-2}
  & \textbf{Pythia}
  & \textbf{Qwen2}
  & \textbf{Llama 3.2} \\
  & \textbf{280M}
  & \textbf{557M}
  & \textbf{410M}
  & \textbf{1.5B}
  & \textbf{1B} \\
\midrule
Parameters      & 280M  & 557M  & 410M  & 1.5B  & 1B    \\
Layers          & 12    & 36    & 24    & 28    & 16    \\
Hidden size     & 1,024 & 1,024 & 1,024 & 1,536 & 2,048 \\
Heads (Q/KV)    & 8/8   & 8/8   & 16/16 & 12/2  & 32/8  \\
Head dim        & 128   & 128   & 64    & 128   & 64    \\
Positional emb. & Abs.  & Abs.  & pRoPE & RoPE  & RoPE  \\
Attention       & MHA   & MHA   & MHA   & GQA ($6\times$) & GQA ($4\times$) \\
Residual        & Seq.  & Seq.  & Par.  & Seq.  & Seq.  \\
\bottomrule
\end{tabular}
\end{table}

\end{document}